\documentclass[twoside,11pt]{article}

\usepackage[submitted]{jmlr}
\usepackage{amsmath}
\usepackage{algorithm}
\usepackage{algorithmic}
\usepackage{multirow}
\usepackage{subfigure}

\newtheorem{assumption}{Assumption}

\jmlrheading{1}{2000}{1-48}{4/00}{10/00}{meila00a}{Shen-Yi Zhao, Hao Gao and Wu-Jun Li}

\firstpageno{1}

\def\g{{\bf g}}

\def\u{{\bf u}}

\def\w{{\bf w}}
\def\x{{\bf x}}

\def\z{{\bf z}}

\def\0{{\bf 0}}
\def\1{{\bf 1}}
\def\2{{\bf 2}}
\def\3{{\bf 3}}
\def\4{{\bf 4}}
\def\5{{\bf 5}}
\def\6{{\bf 6}}
\def\7{{\bf 7}}
\def\8{{\bf 8}}
\def\9{{\bf 9}}

\def\EB{{\mathbb E}}

\def\NB{{\mathbb N}}

\def\RB{{\mathbb R}}

\begin{document}

\title{Quantized Epoch-SGD for Communication-Efficient Distributed Learning}

\author{\name Shen-Yi Zhao \email zhaosy@lamda.nju.edu.cn \\
        \name Hao Gao \email gaoh@lamda.nju.edu.cn \\
        \name Wu-Jun Li \email liwujun@nju.edu.cn \\
        \addr Department of Computer Science and Technology\\
              Nanjing University, China}


\maketitle

\begin{abstract}
  Due to its efficiency and ease to implement, stochastic gradient descent~(SGD) has been widely used in machine learning. In particular, SGD is one of the most popular optimization methods for distributed learning. Recently, quantized SGD~(QSGD), which adopts quantization to reduce the communication cost in SGD-based distributed learning, has attracted much attention. Although several QSGD methods have been proposed, some of them are heuristic without theoretical guarantee, and others have high quantization variance which makes the convergence become slow. In this paper, we propose a new method, called {\underline{Q}}uantized {\underline{E}}poch-{\underline{SGD}}~(QESGD), for communication-efficient distributed learning. QESGD compresses~(quantizes) the parameter with variance reduction, so that it can get almost the same performance as that of SGD with less communication cost. QESGD is implemented on the Parameter Server framework, and empirical results on distributed deep learning show that QESGD can outperform other state-of-the-art quantization methods to achieve the best performance.
\end{abstract}

\section{Introduction}
Many machine learning problems can be formulated as the following optimization problem:
\begin{align}\label{eq:obj}
  \mathop{\min}_{\w\in \RB^d} f(\w) := \frac{1}{n}\sum_{i=1}^{n} f_i(\w).
\end{align}
In~(\ref{eq:obj}), $\w$ refers to the model parameter, $n$ is the number of training data, and each $f_i(\w)$ is the loss function defined on the $i$th instance. For example, given the labeled training data $\{\x_i,y_i\}_{i=1}^n$, if we set $f_i(\w) = \log(1+exp(-y_i\x_i^T\w))$, it is known as logistic regression~(LR). Many deep learning models, like ResNet~\citep{DBLP:conf/cvpr/HeZRS16}, can also be formulated as the form in~(\ref{eq:obj}).

Stochastic gradient descent~(SGD) has been one of the most powerful optimization methods to solve~(\ref{eq:obj}). In the $t$th iteration, SGD randomly selects one mini-batch training data indexed with $B_t$~\citep{DBLP:conf/kdd/LiZCS14} and update the parameter as follows:
\begin{align}\label{eq:SGD}
  \w_{t+1} = \w_t - \frac{\eta_t}{|B_t|}\sum_{i\in B_t} \nabla f_{i}(\w_t),
\end{align}
where $\w_t$ is the parameter value at the $t$th iteration, $B_t$ is the mini-batch sampled at the $t$th iteration, and $\eta_t$ is the learning rate.

Recently, several variants of SGD~\citep{DBLP:journals/jmlr/Shalev-Shwartz013,DBLP:conf/nips/Johnson013,DBLP:conf/nips/DefazioBL14,DBLP:journals/mp/SchmidtRB17} have been proposed and have achieved better performance than traditional SGD in~(\ref{eq:SGD}) for some cases like linear models. However, for some other cases like deep learning models, these variants are not necessarily better than traditional SGD. Hence, the formulation in~(\ref{eq:SGD}) is still the most widely used optimization methods for general machine learning. Furthermore, the SGD in~(\ref{eq:SGD}) is also easy to be implemented on distributed platforms such as Parameter Server: each worker calculates a mini-batch gradient and sent it to the server; server aggregates these gradient and updates the parameters. Hence, SGD is one of the most popular optimization methods for distributed learning.

Recently, quantization has attracted much attention since it can reduce the storage of data and model, cost of computation and communication for distributed learning~\citep{DBLP:conf/icml/Zhang0KALZ17}. Researchers have proposed many methods to combine quantization and SGD. In particular, for training neural networks, many heuristic methods have been proposed~\citep{DBLP:conf/icml/GuptaAGN15,DBLP:conf/icml/ChenWTWC15,DBLP:journals/corr/HubaraCSEB16,DBLP:conf/eccv/RastegariORF16,DBLP:conf/emnlp/AjiH17,lin2018deep} which can quantize parameters, activations and gradients during the training procedure. Most of these methods are heuristic without theoretical guarantee.

Recently, the authors of~\citep{DBLP:conf/nips/WenXYWWCL17,DBLP:conf/nips/AlistarhG0TV17} propose quantized SGD~(QSGD) by compressing gradients with unbiased guarantee. Using previous theory of SGD~\citep{DBLP:conf/nips/BachM11,DBLP:conf/icml/Shamir013}, the methods in~\citep{DBLP:conf/nips/WenXYWWCL17,DBLP:conf/nips/AlistarhG0TV17} converge well. For distributed learning, they only need to communicate low precision gradients. It can save much communication cost which is one of the biggest bottlenecks in distributed learning. The method in~\citep{DBLP:conf/icml/Zhang0KALZ17} tries to compress the training data by executing multiple independent quantizations to make QSGD efficient. However, to get an unbiased quantization vector, all of the above methods will introduce extra variance. Although the authors in~\citep{DBLP:conf/icml/Zhang0KALZ17} propose optimal quantization to reduce the variance, the variance still exists. Combining with the natural variance of stochastic gradients, these algorithms may not perform as well as SGD and it seems hard to reduce the variance asymptotically when compressing gradients. To further reduce the quantization variance, researchers recently propose to compress parameter instead of gradients~\citep{DBLP:journals/corr/abs-1803-03383,DBLP:journals/corr/abs-1803-06443}. The method in~\citep{DBLP:journals/corr/abs-1803-03383} focuses on a variant of SGD called SVRG~\citep{DBLP:conf/nips/Johnson013}. In each epoch, it needs to pass through the training data three times, which is quite slow and not efficient for some models like deep learning models. The method in~\citep{DBLP:journals/corr/abs-1803-06443} focuses on decentralized distributed framework.

In this paper, we propose a new quantized SGD method, called {\underline{Q}}uantized {\underline{E}}poch-{\underline{SGD}}~(QESGD), for communication-efficient distributed learning. QESGD adopts quantization on epoch-SGD~\citep{DBLP:journals/jmlr/HazanK14,DBLP:conf/icml/XuLY17}. QESGD compresses~(quantizes) the parameter with variance reduction, so that it can get almost the same performance as that of SGD with less communication cost. Comparing to existing QSGD methods which need to decrease learning rate after each iteration or set a quite small constant learning rate, QESGD only need to decrease learning rate after one or more epochs. Hence, the changing of learning rate is more similar to the successful practical procedure taken by existing deep learning platforms like Pytorch and Tensorflow. QESGD is implemented on the Parameter Server framework, and empirical results on distributed deep learning show that QESGD can outperform other state-of-the-art quantization methods to achieve the best performance.

\section{Preliminary}
In this paper, we use $\w^*$ to denote the optimal solution of~(\ref{eq:obj}) and use $\|\cdot\|$ to denote the $L_2$-norm. We also make the following common assumptions throughout the paper.

\begin{assumption}\label{ass:smooth}
	We assume that each $f_i(\w)$ is $L$-smooth ($L>0$), which means $\forall \w,\w'$,
	\begin{align*}
		\|\nabla f_i(\w) - \nabla f_i(\w')\| \leq L\|\w - \w'\|.
	\end{align*}
\end{assumption}

\begin{assumption}\label{ass:stronlycovnex}
	We assume that each $f_i(\w)$ is $\mu$-strongly convex ($\mu>0$), which means $\forall \w,\w'$,
	\begin{align*}
		\|\nabla f_i(\w) - \nabla f_i(\w')\| \geq \mu\|\w - \w'\|.
	\end{align*}
\end{assumption}

\begin{assumption}\label{ass:secondmoment}
	The second moment of $\nabla f_i(\w)$ is bounded, which means $\exists G>0$ such that $\forall \w$, $\EB_i[\|\nabla f_i(\w)\|^2|\w] \leq G^2$.
\end{assumption}

\subsection{Quantization}
For simplicity, we use uniform quantization~\citep{DBLP:conf/nips/AlistarhG0TV17,DBLP:journals/corr/abs-1803-03383} in this paper. For any scalar $v\in \RB$, we use $Q_{\delta,b}(v)$ to denote the quantization result of $v$, where $\delta \in \RB, b\in \NB^{+}$,
\begin{align}
	Q_{\delta,b}(v) \in D_{\delta,b} \overset{\triangle}{=} \{\delta \cdot k| k\in \NB, -2^{b-1}\leq k \leq (2^{b-1}-1)\}
\end{align}
and
\begin{align}\label{eq:quantization}
	Q_{\delta,b}(v) = \left\{
	\begin{aligned}
	    &\delta\cdot (2^{b-1}-1) \mbox{ if } v \geq \delta\cdot(2^{b-1}-1) \\
		&\delta \cdot \lfloor\frac{v}{\delta}\rfloor \mbox{ with probability } (\delta \cdot \lceil\frac{v}{\delta}\rceil - v)/\delta \mbox{ if } v \in D_{\delta,b}\\
		&\delta \cdot \lceil\frac{v}{\delta}\rceil \mbox{ with probability } (v - \delta \cdot \lfloor\frac{v}{\delta}\rfloor)/\delta \mbox{ if } v \in D_{\delta,b}\\
		&-\delta\cdot 2^{b-1} \mbox{ if } v \leq -\delta\cdot 2^{b-1}
	\end{aligned}
	\right.
\end{align}
For any vector $\u \in \RB^d$, we also use $Q_{\delta,b}(\u)$ to denote the quantization result where each coordinate of $\u$ is quantified according to (\ref{eq:quantization}) independently. Then we have the following lemma about quantization variance:
\begin{lemma}\label{lem:quantization}
	Given fixed scalars $v \in D_{\delta,b}$ and $v'\in \RB$, we have
	\begin{align}
		\EB[(Q_{\delta,b}(v') - v)^2] \leq (v' - v)^2 + \frac{\delta^2}{4}
	\end{align}
\end{lemma}
\begin{proof}
	The proof is similar to that of~\citep{DBLP:journals/corr/abs-1803-03383}. If $v' \notin D_{\delta,b}$, it is easy to get that
	\begin{align*}
		(Q_{\delta,b}(v') - v)^2 \leq (v' - v)^2
	\end{align*}
	If $v' \in D_{\delta,b}$, let $a = \delta\cdot\lceil\frac{v'}{\delta}\rceil$, according to (\ref{eq:quantization}), we obtain
	\begin{align*}
		     & \EB[(Q_{\delta,b}(v') - v)^2] \\
		=    & \EB[(Q_{\delta,b}(v') - v')^2 + 2(Q_{\delta,b}(v') - v')(v'-v) + (v'-v)^2] \\
		=    & (v' - v)^2 + \EB[(Q_{\delta,b}(v') - v')^2] \\
		=    & (v' - v)^2 + (a - v')^2(v'-a+\delta)/\delta + (a -\delta - v')^2(a - v')/\delta \\
		=    & (v' - v)^2 + (a - v')(v'-a+\delta) \\
		\leq & (v' - v)^2 + \frac{\delta^2}{4}
	\end{align*}
\end{proof}

%

Although such a quantization is a biased estimation ($\EB[(Q_{\delta,b}(v')]\neq v'$), Lemma \ref{lem:quantization} implies that the quantization operation defined in~(\ref{eq:quantization}) would not make the result be far away from the set $D_{\delta,b}$ and if $v' \in D_{\delta,b}$, the quantization variance $\EB[(Q_{\delta,b}(v') - v')^2]$ can be bounded by $\frac{\delta^2}{4}$.

\subsection{Epoch SGD and Motivation}
The Epoch-SGD~\citep{DBLP:journals/jmlr/HazanK14} is presented in Algorithm \ref{alg:epochsgd}. Epoch-SGD updates parameter using a fixed learning rate in each epoch. After each epoch, it will decrease the learning rate and increase $K_t$. Such a training procedure is more practical when comparing to that in~(\ref{eq:SGD}) since the learning rate would descend quickly to zero and it is hard to get a good result.

According to Algorithm \ref{alg:epochsgd}, we can consider the $t$th inner iteration as optimizing the following sub-problem:
\begin{align}\label{eq:subproblem}
	\mathop{\min}_{\z\in \RB^d} g_t(\z) := \frac{1}{n}\sum_{i=1}^n f_i(\w_t + \z)
\end{align}
using SGD with initialization $\0$ and a fixed learning rate. Although it can not get an optimal solution using a fixed learning rate, it can get a good estimation. Furthermore, if $\w_t$ gets close to $\w^*$, then the optimal solution of (\ref{eq:subproblem}) would get close to $\0$. Then we can use the bit centering technique~\citep{DBLP:journals/corr/abs-1803-03383} that compress the variable $\z$ which refers to $\u_{t,k} - \w_t$ in Algorithm \ref{alg:epochsgd}.

\begin{algorithm}[!t]
\caption{Epoch SGD}
\label{alg:epochsgd}
\begin{algorithmic}
\STATE Initialization: $\w_0$;
\FOR{$t=0,1,2, \ldots,T$}
\STATE Let $\u_{t,0} = \w_t$;
\FOR{$k=0$ to $K_t-1$}
\STATE Randomly pick up an instance with index $i_{t,k}$;
\STATE $\u_{t, k+1} = \u_{t,k} - \eta_t \nabla f_{i_{t,k}}(\u_{t,k})$;
\ENDFOR
\STATE Take $\w_{t+1} = \frac{1}{K_t}\sum_{k=0}^{K_t-1}\u_{t,k}$;
\ENDFOR
\end{algorithmic}
\end{algorithm}

\section{QESGD}
Now we present our new quantized SGD called QESGD in Algorithm~\ref{alg:QESGD}. In the $t$th inner iteration, it will update variable $\z$ using the stochastic gradient, and then compress it according to (\ref{eq:quantization}). Using the quantization vector, it recovers real model parameter $\u$ and turns to the next update. For the choice of parameters $K_t, \eta_t, \delta_t, b_t$, we will give details in the later section which leads to the convergence of QESGD.

QESGD is also easy to implemented on Parameter Server. The distributed version of QESGD is presented in Algorithm \ref{alg:disQESGD}. Servers will send quantization vector to workers which will reduce much communication cost. For convergence guarantee and asymptotic reduction of quantization variance in theory, we do not compress the gradients in Algorithm \ref{alg:disQESGD}. In practice, users can compress gradients carefully so that when workers push the gradients, it can also reduce the communication cost. For example, in the experiments of \cite{DBLP:conf/nips/AlistarhG0TV17}, the authors split the vector $\g\in \RB^d$ into $k$ buckets and compress the $k$ buckets individually which can reduce the quantization variance. QESGD can also use this trick when compressing variable $\z$.

\begin{algorithm}[!t]
\caption{QESGD}
\label{alg:QESGD}
\begin{algorithmic}
\STATE Initialization: $\w_0$;
\FOR{$t=0,1,2, \ldots,T$}
\STATE Let $\z_{t,0} = \0, \u_{t,0} = \w_t$;
\FOR{$k=0$ to $K_t-1$}
\STATE Randomly pick up an instance with index $i_{t,k}$;
\STATE $\hat{\z}_{t, k+1} = \z_{t,k} - \eta_t \nabla f_{i_{t,k}}(\u_{t,k})$;
\STATE $\z_{t,k+1} = Q_{\delta_t,b_t}(\hat{\z}_{t, k+1})$;
\STATE $\u_{t,k+1} = \w_t + \z_{t,k+1}$;
\ENDFOR
\STATE Take $\w_{t+1} = \frac{1}{K_t}\sum_{k=0}^{K_t-1}\u_{t,k}$;
\ENDFOR
\end{algorithmic}
\end{algorithm}

\begin{algorithm}[!thb]
\caption{Distributed QESGD}
\label{alg:disQESGD}
\begin{algorithmic}
\STATE Initialization: $\w_0$ on all workers and servers, mini-batch size $B$;
\STATE \textbf{Task of servers}:
\FOR{$t=0,1,2, \ldots,T$}
\STATE Let $\z = \0,\u = \w_t$;
\FOR{$k=0$ to $K_t-1$}
\STATE Wait until receiving vectors $\g_1, \g_2, \ldots, \g_p$ from $p$ workers
\STATE $\z \leftarrow \z - \frac{\eta_t}{B}\sum_{i=1}^p \g_i$;
\STATE $\z \leftarrow Q_{\delta_t,b_t}(\z)$;
\STATE Send $\z$ to all workers;
\ENDFOR
\STATE Take $\w_{t+1} = \w_t + \frac{1}{K_t}\sum_{k=0}^{K_t-1}\z_{t,k}$ and sent it to all workers;
\ENDFOR
\STATE \textbf{Task of workers}
\FOR{$t=0,1,2, \ldots,T$}
\STATE Wait until receiving $\w_t$ from servers;
\FOR{$k=0$ to $K_t-1$}
\STATE Wait until receive quantization vectors $\z$ from servers;
\STATE Randomly pick up a mini-batch instances indexed with $I_p$;
\STATE $\g_p = \sum_{i\in I_p}\nabla f_i(\w_t + \z)$;
\STATE Send $\g_p$ to servers;
\ENDFOR
\ENDFOR
\end{algorithmic}
\end{algorithm}

\section{Convergence analysis}
In this section, we give convergence analysis of QESGD and give details about choosing the parameters $K_t, \eta_t, b_t, \delta_t$ in Algorithm \ref{alg:QESGD}. First, let $\z_t^* = \mathop{\arg\min}_\z g_t(\z)$ where $g_t(\z)$ is defined in (\ref{eq:subproblem}), then we have the lemma:
\begin{lemma}\label{lem:delta}
  Let $\delta_t = \frac{\|\nabla F(\w_t)\|}{\mu2^{b_t-1}}$, then $\|\z_t^*\|_\infty \in D_{\delta_t,b_t}$.
\end{lemma}
\begin{proof}
  By the definition of $\z_t^*$, we obtain $\|\z_t^*\| = \|\w_t - \w^*\|\leq \frac{1}{\mu}\|\nabla F(\w_t)\| = \delta_t 2^{b_t-1}$. It implies that each coordinate of $\z_t^*$ belongs to $D_{\delta_t,b_t}$.
\end{proof}

\begin{theorem}
  Let $\{\u_{t,k}\}, \{\z_{t,k}\}$ be the sequences in Algorithm \ref{alg:QESGD}. With Assumption \ref{ass:smooth}, \ref{ass:stronlycovnex} and \ref{ass:secondmoment}, $\delta_t = \frac{\|\nabla F(\w_t)\|}{\mu2^{b_t-1}}$, $\kappa = \frac{L}{\mu}$, we have the following result
  \begin{align}
       \EB[F(\w_{t+1})- F(\w^*)] \leq (\frac{1}{\mu\eta_tK_t} + \frac{\kappa d}{\mu\eta_t2^{2b_t}})\EB(F(\w_t) - F(\w^*)) + \frac{\eta_tG^2}{2} \nonumber
\end{align}
\end{theorem}
\begin{proof}
Let $\z_t^* = \mathop{\arg\min}_\z g_t(\z)$, where $g_t(\z)$ is defined in (\ref{eq:subproblem}). Then we have $\w_t + \z_t^* = \w^*$.
\begin{align}
	     & \EB[\|\u_{t,k+1} - \w^*\|^2|\u_{t,k}] \nonumber \\
	=    & \EB[\|\w_t + Q_{\delta_t,b_t}(\hat{\z}_{t,k+1}) - \w^*\|^2|\u_{t,k}] \nonumber \\
    =    & \EB[\|Q_{\delta_t,b_t}(\hat{\z}_{t,k+1}) - \z_t^*\|^2|\u_{t,k}] \nonumber \\
    \leq & \EB[\|\hat{\z}_{t,k+1} - \z_t^*\|^2|\u_{t,k}] + \frac{d\delta_t^2}{4} \nonumber \\
    =    & \EB[\|\z_{t,k} - \eta_t\nabla f_{i_{t,k}}(\u_{t,k}) - \z_t^*\|^2|\u_{t,k}] + \frac{d\delta_t^2}{4} \nonumber \\
    =    & \EB[\|\u_{t,k} - \eta_t\nabla f_{i_{t,k}}(\u_{t,k}) - \w^*\|^2|\u_{t,k}] + \frac{d\delta_t^2}{4} \nonumber \\
    =    & \EB[\|\u_{t,k} - \w^*\|^2 - 2\eta_t\nabla f_{i_{t,k}}(\u_{t,k})^T(\u_{t,k} - \w^*) +  \eta_t^2\|\nabla f_{i_{t,k}}(\u_{t,k})\|^2|\u_{t,k}] + \frac{d\delta_t^2}{4} \nonumber \\
    \leq & \|\u_{t,k} - \w^*\|^2 - 2\eta_t\nabla F(\u_{t,k})^T(\u_{t,k} - \w^*) +  \eta_t^2G^2 + \frac{d\delta_t^2}{4} \nonumber
\end{align}
The first inequality uses Lemma \ref{lem:quantization} and Lemma \ref{lem:delta}. The last inquality uses the the fact that $\EB[\nabla f_{i_{t,k}}(\u_{t,k})|\u_{t,k}] = F(\u_{t,k})$. By the convexity of $F(\w)$, we get that $F(\u_{t,k}) - F(\w^*) \leq \nabla F(\u_{t,k})^T(\u_{t,k} - \w^*)$. Then we obtain
\begin{align}
       & \EB[F(\u_{t,k}) - F(\w^*)] \nonumber \\
  \leq & \frac{1}{2\eta_t}(\EB[\|\u_{t,k} - \w^*\|^2 - \EB[\|\u_{t,k+1} - \w^*\|^2) + \frac{\eta_tG^2}{2} + \frac{d\delta_t^2}{8\eta_t}
\end{align}
Summing up the above equation from $k=0$ to $K_t-1$ and taking $\w_{t+1} = \frac{1}{K_t}\sum_{k=0}^{K_t-1}\u_{t,k}$, we obtain
\begin{align}
       & \EB[F(\w_{t+1})- F(\w^*)] \nonumber \\
  \leq & \frac{1}{2\eta_tK_t}\|\w_t - \w^*\|^2 + \frac{\eta_tG^2}{2} + \frac{d\delta_t^2}{8\eta_t} \nonumber \\
  \leq & \frac{1}{\mu\eta_tK_t}(F(\w_t) - F(\w^*)) + \frac{\eta_tG^2}{2} + \frac{d\delta_t^2}{8\eta_t} \nonumber \\
  \leq & \frac{1}{\mu\eta_tK_t}(F(\w_t) - F(\w^*)) + \frac{\eta_tG^2}{2} + \frac{Ld(F(\w_t) - F(\w^*))}{\eta_t\mu^22^{2b_t}} \nonumber \\
  =    & (\frac{1}{\mu\eta_tK_t} + \frac{\kappa d}{\mu\eta_t2^{2b_t}})(F(\w_t) - F(\w^*)) + \frac{\eta_tG^2}{2} \nonumber
\end{align}
where the last inequality uses the smooth property that $\|\nabla F(\w)\|^2 \leq 2L(F(\w) - F(\w^*)), \forall \w$.
\end{proof}

Now let's make details on the choice of parameters for finial convergence.
\begin{corollary}
  Let $K_t = \frac{1}{3\mu\eta_t}$, $b_t = \log(\sqrt{\kappa d K_t})$, then we have
  \begin{align*}
    \EB[F(\w_{t+1})- F(\w^*)] \leq \frac{2}{3}\EB(F(\w_t) - F(\w^*)) + \frac{\eta_tG^2}{2}
  \end{align*}
  Moreover, let $\eta_t = \mathcal{O}(1/t)$, then $\EB(F(\w_t) - F(\w^*)) \leq \mathcal{O}(1/t)$.
\end{corollary}
\begin{proof}
	First, it is easy to calculate $\frac{1}{\mu\eta_tK_t} + \frac{\kappa d}{\mu\eta_t2^{2b_t}} = \frac{2}{3}$. For convenience, let $y_t = \EB(F(\w_t) - F(\w^*))$, $\eta_t = \frac{2c}{G^2t}$, where $c>0$ is a constant. Then we have
	\begin{align}
		y_{t+1} \leq \frac{2}{3}y_t + \frac{c}{t}
	\end{align}
	We proof the result by induction. Assuming a constant $A$ satisfies $y_4 \leq \frac{A}{4}$ and $A \geq \frac{15c}{2}$. If $y_t \leq \frac{A}{t} (t\geq 4)$, then
	\begin{align}
		y_{t+1} \leq \frac{2A+3c}{3t} \leq \frac{A}{t+1}
	\end{align}
	Above all, we get that $y_t \leq \mathcal{O}(1/t)$.
\end{proof}

On the choice of $\delta_t$, it is related to the full gradient $\|\nabla F(\w_t)\|$. Computing the full gradient is unacceptable since the large scale training data. In fact, $\delta_t \propto \|\nabla F(\w_t)\|$ and according to the corollary, we obtain $\|\nabla F(\w_t)\|^2 \leq 2L(F(\w_t) - F(\w^*)) \leq \mathcal{O}(1/t)$. This implies that we can directly set $\delta_t = \mathcal{O}(\frac{1}{\sqrt{t}2^{b_t-1}})$ so that we can avoid the full gradient computation. At the same time, it implies that the quantization variance would decrease to zero. On the choice of $b_t$, since $\eta_t = \mathcal{O}(1/t)$, we get that $b_t = \mathcal{O}(\log(\sqrt{t}))$, which would increase quite slowly as $t$ increases (when $t$ is large). Within finite training time, we can consider it as a constant. In our experiments, we find that setting $b_t = 8$ is good enough which would lead to the same performance as that of SGD.

\section{Experiments}
We use deep neural networks to evaluate QESGD. We do experiments on Pytorch using TITAN xp GPU. We compare our method with SGD and QSGD. Since the distributed version of these methods take synchronous strategy, the performance is equivalent to that on single machine. In this paper, we would only do experiments on single machine to verify the impact of quantization on training and testing results. To evaluate the efficiency of variance reduction of quantization in our method, we would compress the whole vector directly using uniform quantization without any other tricks for both QESGD and QSGD.

\textbf{CNN.} First, we choose two CNN models: ResNet-20 and ResNet-56. We use the data set CIFAR10. For QESGD, we set $\delta_t = \frac{\|\nabla F(\w_0)\|}{c\sqrt{t}2^{b-1}}$ (We only calculate the full gradient w.r.t the initialization $\w_0$), where the constant $c$ is chosen from $\{1,2,3,4,5,10\}$. For QSGD, we set $\delta = \|\g\|$, where $\g$ is the gradient that need to be compressed. The learning rates of QESGD and QSGD are the same as that of SGD. The result is in Figure \ref{exp:cnn}. We can find that QESGD gets almost the same performance on both training and testing results as that of SGD. Due to the quantization variance, QSGD is weak. The gap between QSGD and SGD is pronounced. We also train ResNet-18 on imagenet, the result is in Figure \ref{exp:imagenet}.

Many evidences have show that weight decay would affect the distribution of model parameters and gradients. Since we take uniform quantization,  weight decay would affect the quantization variance. The number of bits can also affect the quantization variance. Then we evaluate these methods on the large model ResNet-56 under different weight decays and quantization bits. The performance is in Table \ref{tab:cnn}. We can find that: (a) under the same settings, QESGD is always better than QSGD; (b) when we do not use weight decay, quantization method would be a little weak than SGD; (c) when we use small bits, the quantization methods have obviously deteriorated, especially that of QSGD.
\begin{figure}[!htb]
  \centering
  \subfigure[Train on ResNet-20]{\includegraphics[width=2.6in]{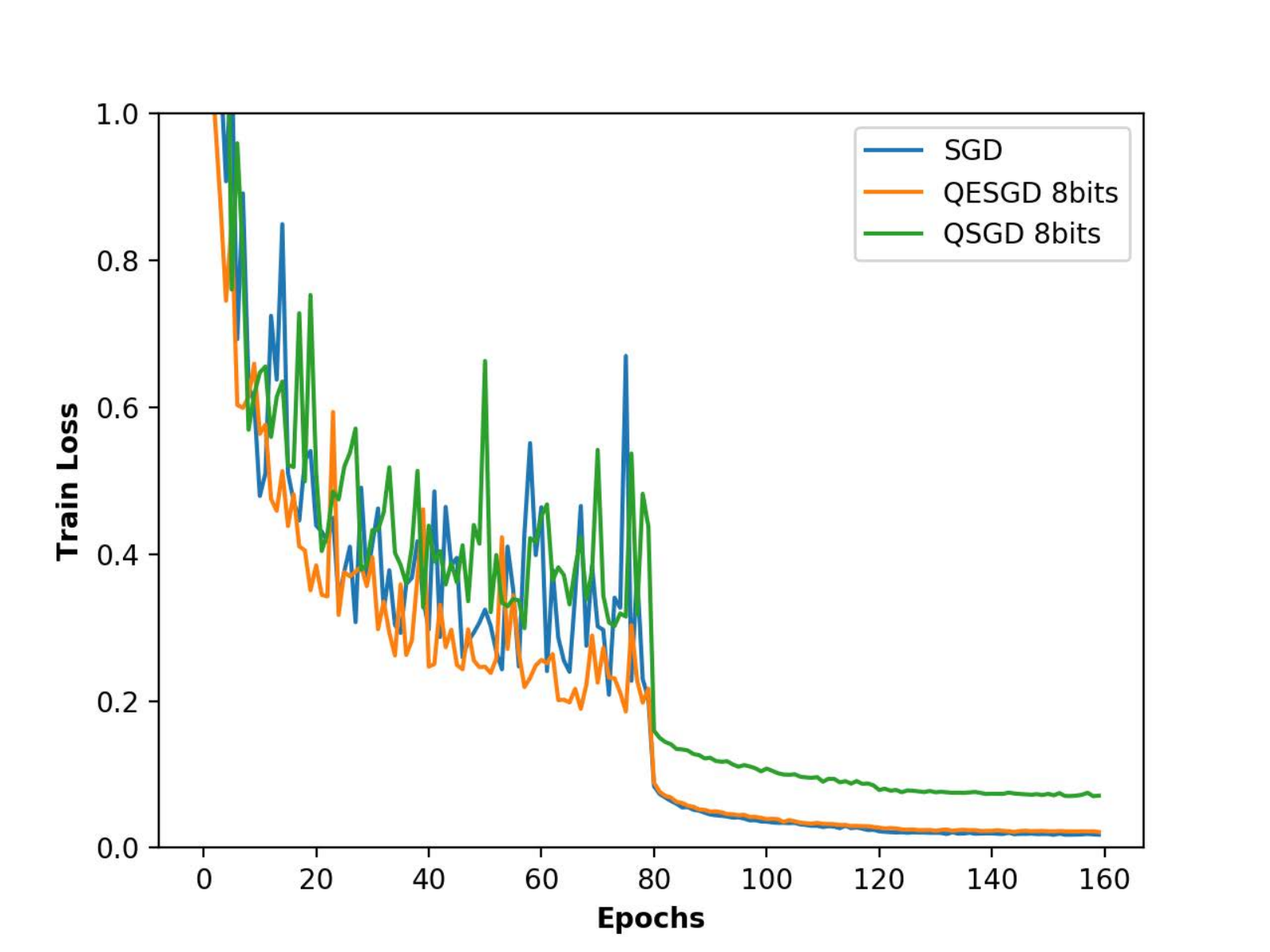}}
  \subfigure[Test on ResNet-20]{\includegraphics[width=2.6in]{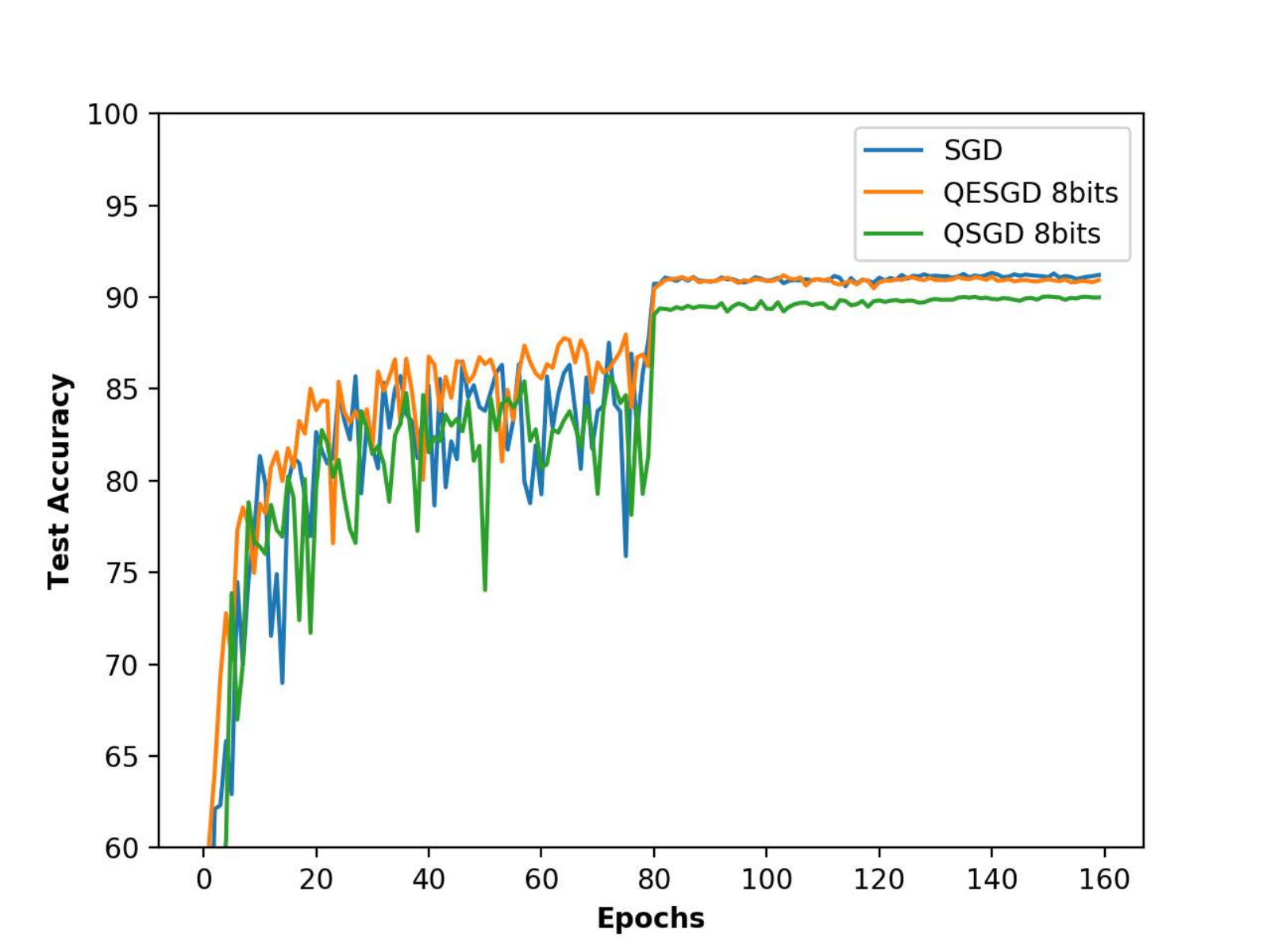}}
  \subfigure[Train on ResNet-56]{\includegraphics[width=2.6in]{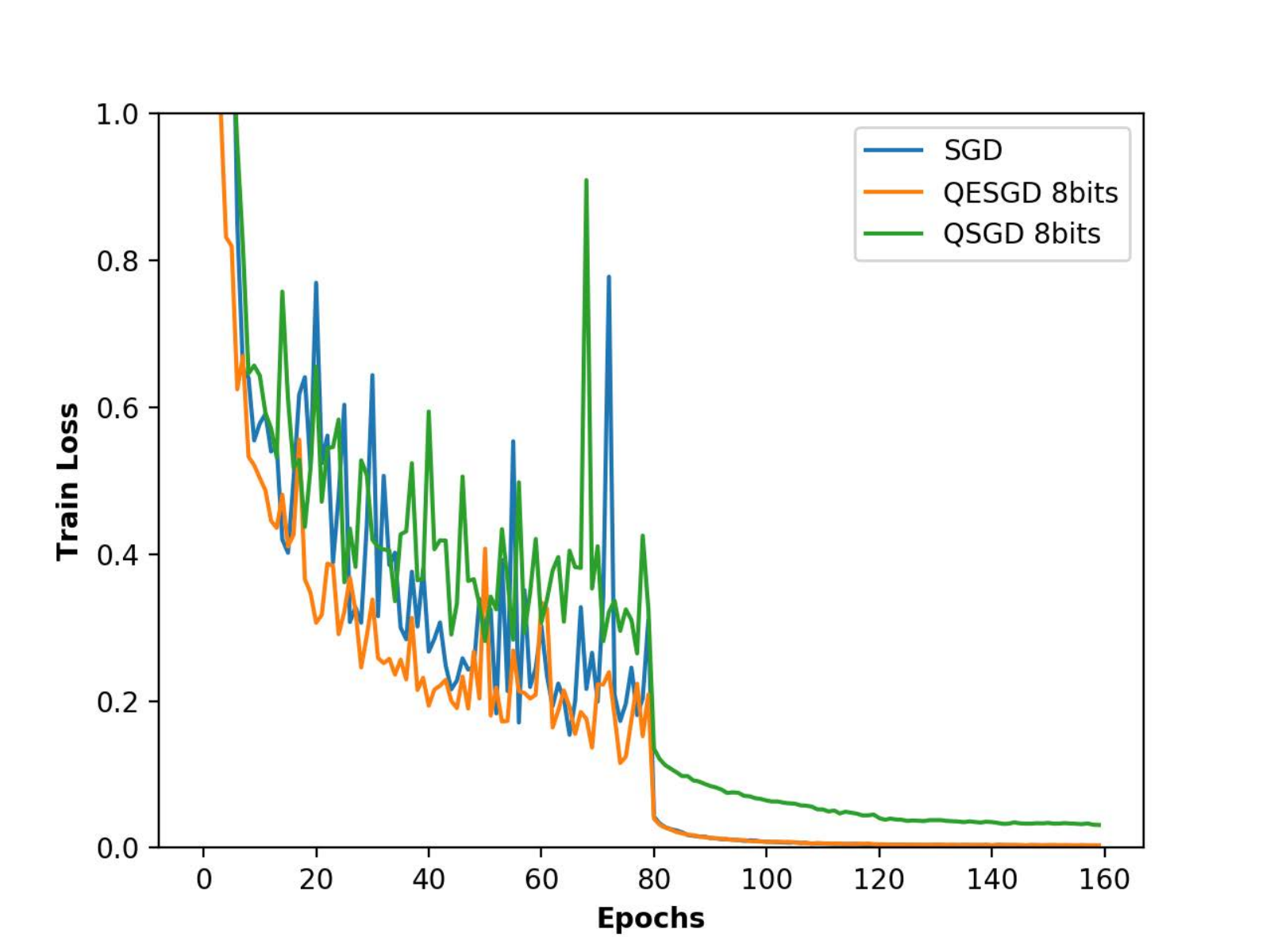}}
  \subfigure[Test on ResNet-56]{\includegraphics[width=2.6in]{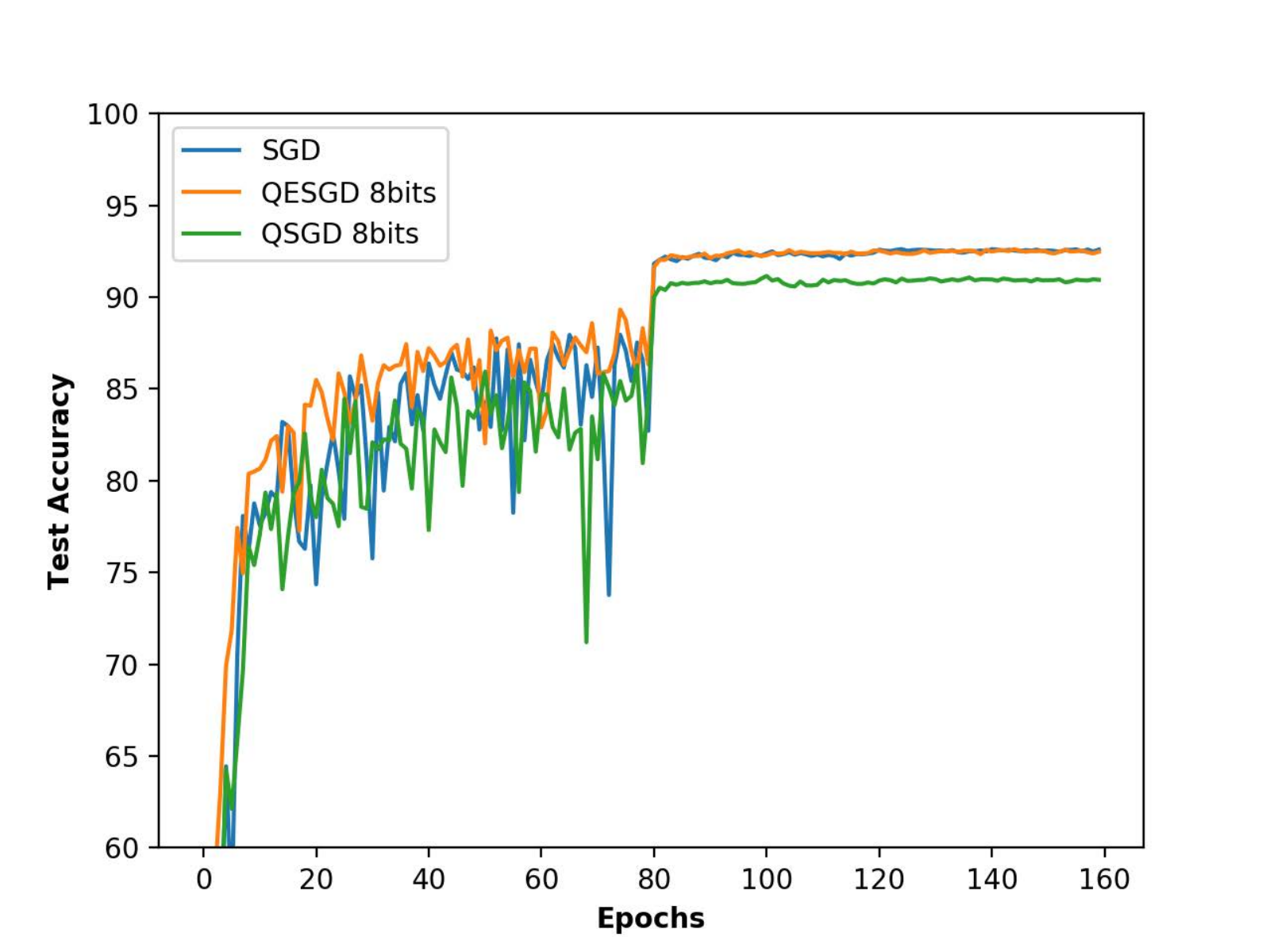}}
  \caption{Efficiency comparison with baselines on CNN models. For QSGD and QESGD, we use 8 bits quantization. The initialization of learning rate is 0.5, the batch size is 128, weight decay is 0.0001.}\label{exp:cnn}
\end{figure}

\begin{table}[!thb]
	\centering
	\begin{tabular}{|c|c|c|c|}
		\hline
		~                      & weight decay & bits    & test accuracy  \\ \hline
		\multirow{2}{*}{SGD}   &     0        &   $-$   &    90.87$\%$   \\ \cline{2-4}
		                       &     0.0001   &   $-$   &    92.61$\%$   \\ \hline
		\multirow{3}{*}{QSGD}  &     0        &   8     &    89.91$\%$    \\ \cline{2-4}
		                       &     0.0001   &   8     &    91.15$\%$   \\ \cline{2-4}
		                       &     0.0001   &   4     &    80.29$\%$   \\ \hline
		\multirow{3}{*}{QESGD} &     0        &   8     &    90.64$\%$         \\ \cline{2-4}
		                       &     0.0001   &   8     &    92.61$\%$   \\ \cline{2-4}
		                       &     0.0001   &   4     &    87.72$\%$   \\ \hline
	\end{tabular}
	\caption{Performance under different settings}\label{tab:cnn}
\end{table}

\begin{figure}[!htb]
  \centering
  \includegraphics[width=2.6in]{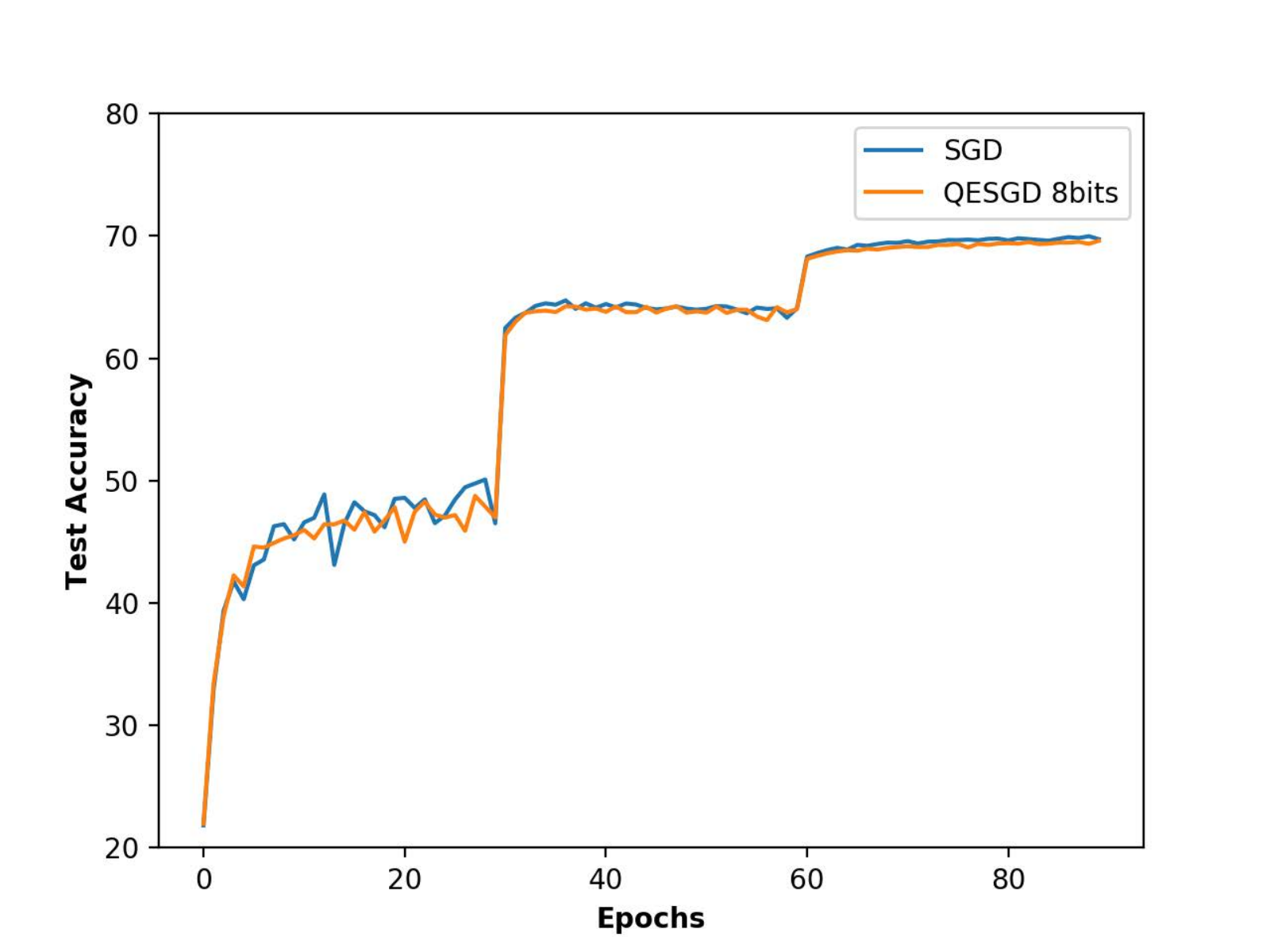}
  \caption{Training ResNet-18 on imagenet. The initialization of learning rate is 0.1, the batch size is 256, weight decay is 0.0001.}\label{exp:imagenet}
\end{figure}

\textbf{RNN.} We also evaluate our method on RNN. We choose the model LSTM that contains two hidden layers, each layer contains with 128 units and the data set TinyShakespeare~\footnote{\url{https://github.com/karpathy/char-rnn}}. The choice of $\delta$ of QESGD and QSGD is the same as that in CNN experiments. The result is in Figure \ref{exp:rnn}. QESGD still gets almost the same performance as that of SGD. Sometimes it is even better than SGD. In this experiment, we can find the gap between QSGD and SGD is smaller than that in CNN experiments. This is due to the gradient clipping technique which is common in the training of RNN. It can reduce the quantization variance of gradients so that QSGD performs well.
\begin{figure}\label{exp:rnn}
\centering
	\includegraphics[width = 6.0cm]{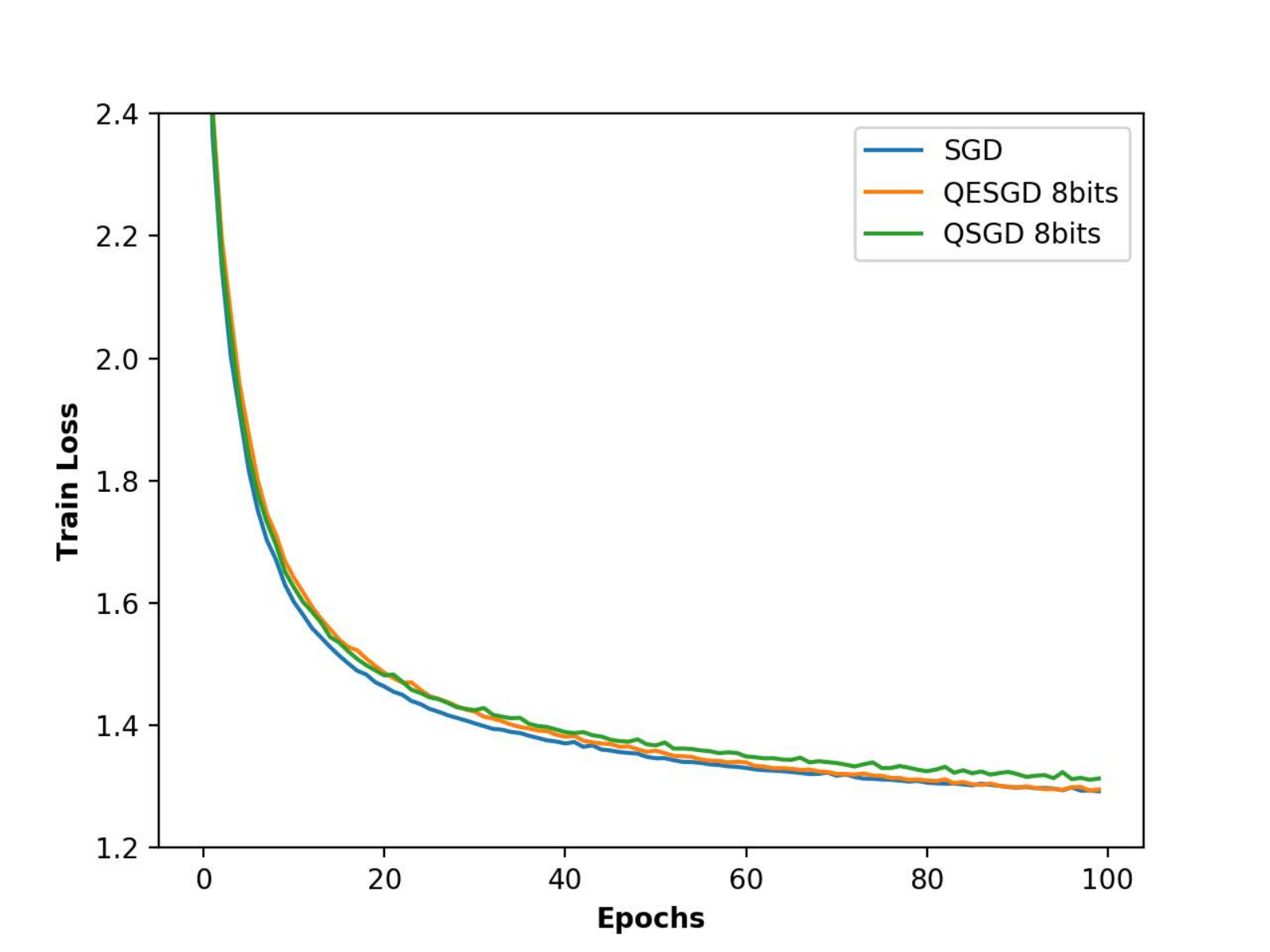}
	\includegraphics[width = 6.0cm]{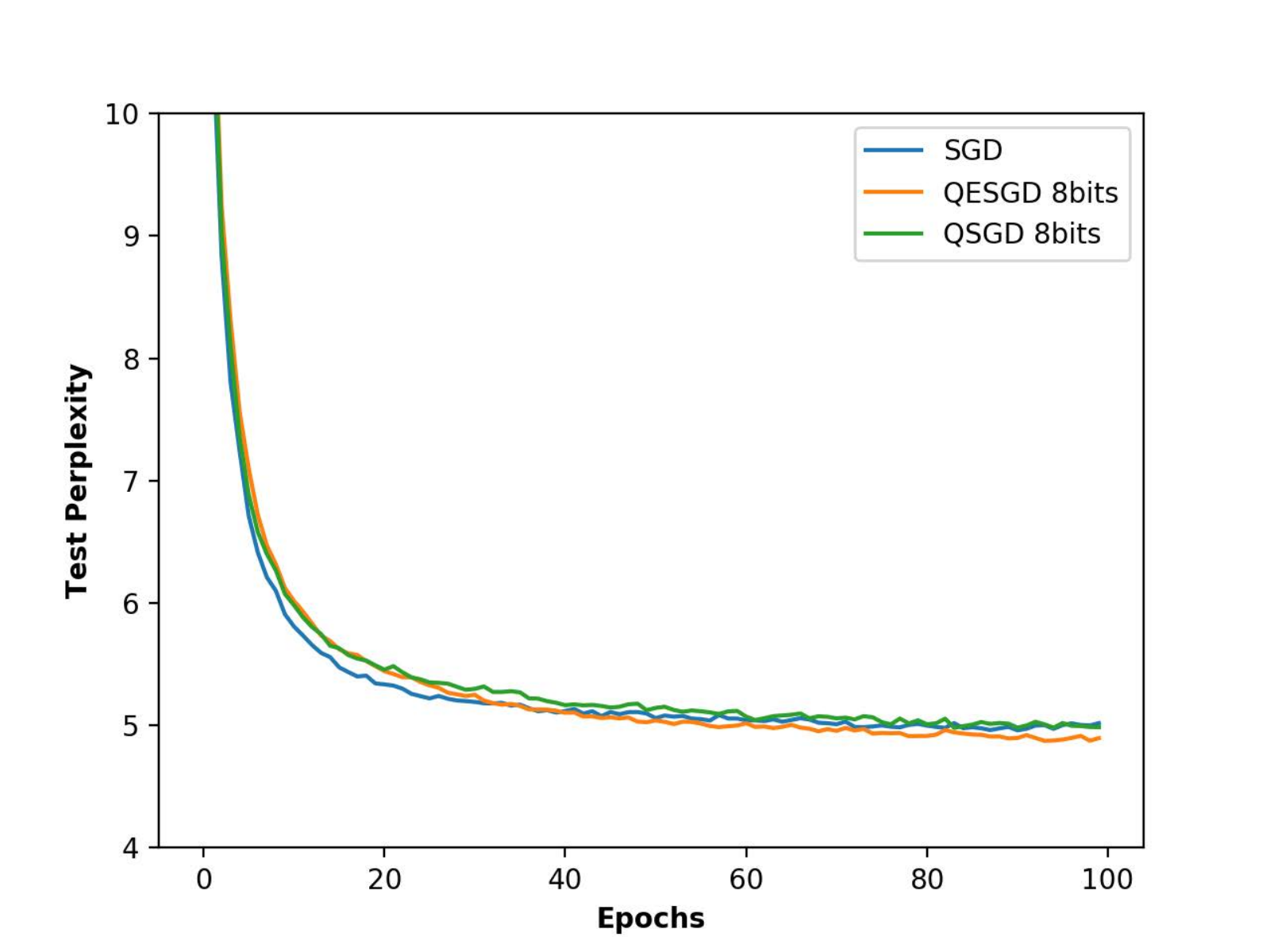}
\caption{Evaluation on RNN model. The initial learning rate is 2, batch size is 50, drop-out ratio is 0.25.}
\end{figure}

\textbf{Distrbuted training}. We also evaluate the communication efficiency of distributed QESGD on
Parameter Server. We conduct experiments on docker with 8 k40 GPUs and 1 server. We use three
models: ResNet-56, AlexNet and VGG-19. The result is in Table 2. The Speedup is defined as (Time
per epoch of SGD)/(Time per epoch of QESGD) under the same number of GPUs. Since QESGD
uses 8bits and SGD uses 32 bits during communication, the ideal speedup is 2/(1+8/32) = 1.6.
Due to the computation cost, the results in Table 2 are smaller than 1.6. On this hand, our method can
reduce communication efficiently.

\begin{table}[!thb]
	\centering
	\begin{tabular}{|c|c|c|c|}
		\hline
		Model                      & Parameters             & GPUs  & Speedup(Ideal 1.6)  \\ \hline
		\multirow{2}{*}{ResNet-56} & \multirow{2}{*}{0.85M} &   4   &    1.12$\times$       \\ \cline{3-4}
		                           &                        &   8   &    1.31$\times$       \\ \hline
		\multirow{2}{*}{AlexNet}   & \multirow{2}{*}{57M}   &   4   &    1.34$\times$       \\ \cline{3-4}
		                           &                        &   8   &    1.49$\times$       \\ \hline
		\multirow{2}{*}{VGG-19}    & \multirow{2}{*}{140M}  &   4   &    1.39$\times$       \\ \cline{3-4}
		                           &                        &   8   &    1.38$\times$       \\ \hline
	\end{tabular}
	\caption{Speedup on different models}\label{tab:speedup}
\end{table}

\section{Conclusion}
In this paper, we propose a new quantization SGD called QESGD. It can reduce the quantization variance by compressing parameters instead of gradients. It is also easy to implemented on distributed platform so that it can reduce the communication by quantization.



\bibliography{ref}
\bibliographystyle{plain}

\end{document}